\documentclass[10pt,twocolumn,letterpaper]{article}
\usepackage[T1]{fontenc}
\usepackage[pagenumbers]{cvpr}

\makeatletter
\patchcmd{\@maketitle}{\vskip .375in}{\vskip 8pt}{}{\PackageWarningNoLine{jome}{Could not adjust preprint title placement}}
\patchcmd{\@maketitle}{\vspace*{24pt}}{\vspace*{10pt}}{}{\PackageWarningNoLine{jome}{Could not adjust preprint author spacing}}
\makeatother
\definecolor{cvprblue}{rgb}{0.21,0.49,0.74}
\usepackage[pagebackref,breaklinks,colorlinks,allcolors=cvprblue]{hyperref}

\usepackage{hyperref}       
\usepackage{url}            
\usepackage{booktabs}       
\usepackage{amsfonts}       
\usepackage{nicefrac}       
\usepackage{microtype}      
\usepackage{xcolor}         
\usepackage{amsthm}
\usepackage{amsmath}
\usepackage{amssymb}
\usepackage{stfloats}
\usepackage{multirow}
\usepackage{graphicx}
\usepackage{graphicx}
\usepackage{wrapfig}
\usepackage{enumitem}
\usepackage{adjustbox}
\usepackage{amsthm}
\usepackage{amsmath}
\usepackage{algorithm}
\usepackage{algorithmic}

\usepackage{xcolor}

\newtheorem{theorem}{Theorem}
\newtheorem{Definition}{Definition}

\title{The Projector Forgets Too: Joint Merging for Multimodal Continual \\ Instruction Tuning}

\author{
    Zijian Gao\textsuperscript{1,2},
    Wangwang Jia\textsuperscript{1},
    Xingxing Zhang\textsuperscript{3},
    Pengfei Qian\textsuperscript{1},
    Tao Sun\textsuperscript{1}, 
    Bo Ding\textsuperscript{1,2},\\
    Yong Dou\textsuperscript{1},
    Zhen Huang\textsuperscript{1},
    Dongsheng Li\textsuperscript{1},
    Huaimin Wang\textsuperscript{1,2},
    Kele Xu\textsuperscript{1,2}\thanks{Corresponding author}\\
    {\normalsize\textsuperscript{1}PDL Lab, College of Computer Science and Technology, National University of Defense Technology}\\
    {\normalsize\textsuperscript{2}State Key Laboratory of Complex \& Critical Software Environment}\\
    {\normalsize\textsuperscript{3}School of Computer Science, Tsinghua University}\\
    {\tt\small \{gaozijian19,wangwangjia,pengfeiqian,yongdou,huangzhen,dsli\}@nudt.edu.cn}\\
    {\tt\small \{hmwang,xukelele\}@nudt.edu.cn},
    {\tt\small xxzhang1993@gmail.com, nudtsuntao@163.com}
}

\begin{document}
\maketitle

\begin{abstract}
Multimodal Continual Instruction Tuning (MCIT) is essential for sequential task adaptation of Multimodal Large Language Models (MLLMs) but is severely restricted by catastrophic forgetting. 
While existing literature focuses on the reasoning language backbone, in this work, we find that forgetting can arise from changes in both the cross-modal projector and the language model's LoRA adapters.
In component restoration experiments, restoring an earlier task's LoRA adapter while retaining the final-task projector can leave an accuracy gap relative to restoring both components.
To address these two sources of forgetting, we present \textbf{JoMe} (\textbf{Jo}int \textbf{Me}rging), which combines \textbf{C}ross-modal \textbf{P}rojection \textbf{M}erging (\textbf{CPM}) and \textbf{L}ow-rank \textbf{P}arameter \textbf{M}erging (\textbf{LPM}).
Specifically, CPM recovers perceptual alignment by adaptively merging cross-modal visual representations via visual-prototype guidance during inference. 
Simultaneously, LPM recursively merges task-specific LoRA updates to mitigate interference across tasks.
By leveraging recursive least squares, LPM yields a closed-form solution to the layer-wise least-squares objective on fixed inputs. Notably, JoMe operates as a training-free paradigm that achieves knowledge merging via efficient CPU-based algebraic operations, eliminating additional gradient-based optimization beyond initial tuning. Our extensive evaluations confirm the superior performance and robustness of JoMe across multiple MLLMs and benchmarks. Specifically, on the UCIT benchmark, JoMe achieves significant leads of up to 5.37\% and 1.72\% in final average accuracy over state-of-the-art methods across two different MLLMs, respectively.
\footnote{Our code and data are included in the supplementary material for examination.}
\end{abstract}

\section{Introduction}

Recent advancements in multimodal large language models (MLLMs) \cite{yin2024survey,bai2023qwen,zhu2023minigpt} have significantly extended traditional LLM \cite{touvron2023llama} capabilities.
To bridge the gap between general-purpose pretraining and domain-specific applications, multimodal instruction tuning \cite{zhang2026instruction,tong2025metamorph} aligns model behavior with user intent. However, real-world scenarios typically require models to continuously adapt to diverse, sequentially arriving datasets to meet evolving demands. 
Multimodal continual instruction tuning (MCIT)~\cite{chen2024coin} requires MLLMs to learn new tasks sequentially, but adapting to new tasks can cause catastrophic forgetting of earlier ones \cite{liu2025continual,gao2025maintaining,gao2025knowledge,gao2025rethinking,GAO2024106513,gao2024stabilizing}.

To combat forgetting, current methods face difficult trade-offs: Mixture of Experts (MoE) based architectures \cite{guo2025hide,yu2025progressive} require heavy computation, regularization \cite{chen2025sefe,zeng2025modalprompt} limits model flexibility, and replay approaches \cite{li2025multimodal,lee2025oasis} raise privacy and storage concerns. Crucially, these methods mainly share a fundamental oversight: they focus almost entirely on the anti-forgetting ability of the language backbone, assuming the multimodal projector is immune to forgetting. 


To examine how the cross-modal projector and the LoRA adapter contribute to forgetting, we conduct a component restoration experiment on UCIT~\cite{guo2025hide}. 
Figure~\ref{fig:motivation_analysis} compares the model obtained after sequential fine-tuning (green) with two variants. In each variant, we restore one component from an earlier task, either its projector (blue) or its LoRA adapter (orange), while keeping the other component at its final-task state.
The task-specific oracle (dashed black) uses both the projector and the LoRA adapter saved for the task being evaluated. With LLaVA~\cite{liu2024improved} and InternVL~\cite{chen2024internvl}, restoring the LoRA adapter alone leaves a gap to the oracle on tasks such as Viz-cap~\cite{gurari2018vizwiz} and IconQA~\cite{lu2021iconqa}. Both configurations use the same task-specific adapter, so the gap shows how much performance is lost by retaining the final-task projector. On CLEVR~\cite{lindstrom2022clevr} with LLaVA, restoring the projector alone yields higher accuracy than restoring the adapter alone. These observations motivate JoMe, which addresses both components and improves over sequential fine-tuning on several evaluated tasks (red).


\begin{figure*}[t]
    \centering
 \includegraphics[width=.75\linewidth]{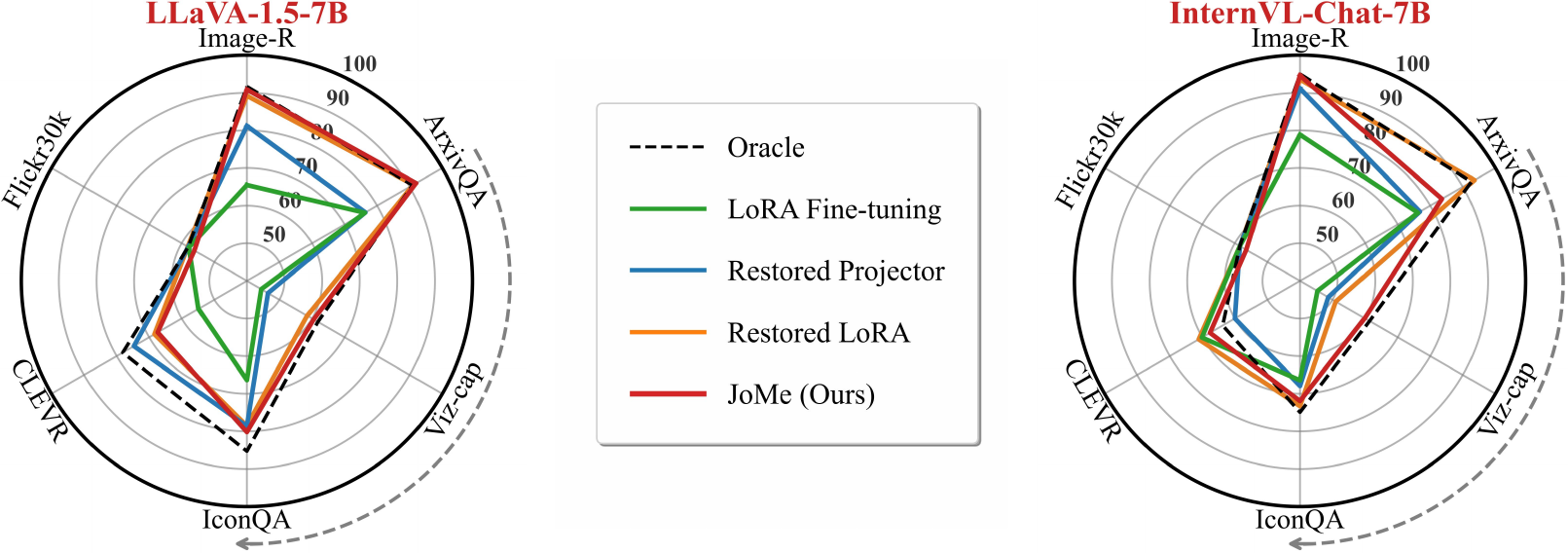}
 \vspace{-0.25cm}
\caption{
    \textbf{Performance diagnosis on UCIT benchmark for LLaVA-1.5-7B and InternVL-Chat-7B using sequential LoRA fine-tuning. Tasks follow a clockwise sequence from Image-R.} 
}
    \label{fig:motivation_analysis}
     \vspace{-0.6cm}
\end{figure*}

Driven by these insights, we introduce \textbf{JoMe} (\textbf{Jo}int \textbf{Me}rging), a simple yet effective framework designed to fix both the perception and reasoning spaces. To reverse the perception drift seen in our experiments, we propose Cross-modal Projection Merging (CPM). Instead of using a single projector, CPM maintains lightweight, task-specific projection layers. These layers are highly parameter-efficient and ensure fast inference by adaptively merging visual features. Guided by visual prototypes, CPM can automatically recognize which task an image belongs to and recover the most accurate visual-language alignment. Simultaneously, to stop the reasoning collapse within the LLM, we develop Low-rank Parameter Merging (LPM). Unlike simple weight averaging that often causes severe parameter interference, LPM uses the proposed Recursive Least Squares algorithm to find a conflict-minimizing way to merge task-specific low-rank modules. This yields a closed-form solution that minimizes the cumulative squared feature drift at each layer on fixed inputs. Crucially, by performing dual-merging via fast CPU-based algebraic operations, JoMe eliminates the need for additional GPU training or historical data, offering a lightweight solution for real-world multimodal learning.

In summary, our key contributions can be summarized as follows:
(1) We show a neglected dual-forgetting phenomenon, proving that projector drift is a major cause of failure alongside LLM forgetting.
(2) We introduce CPM, a plug-and-play module that adaptively merges task-specific visual features to restore perceptual alignment.
(3) We develop LPM, which utilizes a recursive solution based on Recursive Least Squares to achieve conflict-minimizing parameter merging.
(4) The proposed JoMe framework achieves state-of-the-art performance on the UCIT and MLLM-DCL benchmarks, demonstrating strong knowledge retention without any additional GPU-based training.

\section{Related Work}

\textbf{Multimodal Large Language Models.} Driven by the success of LLMs \cite{touvron2023llama,zhang2023llama} in natural language processing \cite{nirenburg1993progress,min2023recent}, MLLMs \cite{bai2023qwen,zhu2023minigpt,dai2023instructblip} have emerged to bridge the gap between vision and language. Architecturally, MLLMs rely on a frozen vision encoder and a cross-modal alignment module (such as a projector \cite{liu2023visual} or a cross-attention mechanism \cite{dai2023instructblip}) to align visual features with the language space. While achieving impressive performance \cite{liu2024mmbench,zhang2024mm,zhao2025chartcoder}, most existing MLLMs are trained in static multi-task settings. They struggle to adapt to continually arriving data streams \cite{shi2021overcoming}, where learning new knowledge often leads to catastrophic forgetting of previously acquired capabilities.

\textbf{Continual Instruction Tuning for MLLMs.} 
Existing solutions for overcoming forgetting in MCIT \cite{liu2023visual,longpre2023flan} generally fall into three categories: replay-based, regularization-based, and architecture-based methods. 
Replay-based methods \cite{li2025multimodal,lee2025oasis} rehearse historical or synthetic data, introducing severe storage/computation overheads and privacy risks.
Regularization-based approaches \cite{zeng2025modalprompt,chen2025sefe} penalize drastic parameter or cross-modal shifts to bypass memory burdens \cite{wang2023orthogonal}.
Architecture-based strategies \cite{guo2025hide,wang2023orthogonal,xie2026same} isolate new knowledge via task-specific modules or dynamic routing. While inherently preserving prior parameter spaces, they face critical challenges in precise task-agnostic routing. Crucially, these paradigms focus almost exclusively on the language backbone, assuming the multimodal projector is immune to forgetting. 
Our work challenges this assumption by revealing and addressing the drift in the projection space.

\textbf{Model Merging.} Model merging consolidates knowledge from multiple task-specific models into a unified architecture without costly retraining. Early works relied on weight averaging \cite{matena2022merging} or feature distance minimization \cite{jin2023dataless}, while recent methods \cite{yadav2023ties,du2024parameter} utilize task vectors \cite{ilharco2022editing} to integrate modules like LoRA \cite{hu2022lora}. To mitigate negative transfer, advanced training-free paradigms \cite{sun2025towards} have begun to explicitly minimize layer-wise feature drift. However, these methods are predominantly tailored for static, one-shot merging within the NLP domain, lacking a mechanism to handle continual learning scenarios. In contrast, our work is the first to bridge model merging with the MCIT paradigm. Moving beyond one-time integration, JoMe adaptively merges task-specific projector outputs at inference and recursively merges LoRA updates as new tasks arrive.


\section{Preliminary}
\label{Preliminary}

\subsection{Multimodal Continual Instruction Tuning}
\label{sec:preliminary_cit}
We consider an MLLM \cite{liu2023visual} composed of three primary components: a vision encoder $f_v$ parameterized by $\Phi_v$, a cross-modal projector $\mathcal{P}$ parameterized by $\Phi_\mathcal{P}$, and a large language model (LLM). In the context of MCIT, the model is required to learn a sequence of $n$ tasks, denoted as $\mathcal{T} = \{1, 2, \dots, n\}$, which arrive sequentially. For each task $t \in \mathcal{T}$, the task dataset is defined as $\mathcal{D}_t = \{(x_v^{t,j}, x_{ins}^{t,j}, x_{ans}^{t,j})\}_{j=1}^{N_t}$, where $x_v^{t,j}$, $x_{ins}^{t,j}$, and $x_{ans}^{t,j}$ represent the input raw image, instruction tokens, and answer tokens, respectively, and $N_t$ is the number of samples in $\mathcal{D}_t$.

For a given sample $(x_v, x_{ins}, x_{ans}) \in \mathcal{D}_t$, the vision encoder first extracts visual features from the raw image. Specifically, we follow the prevailing paradigm \cite{liu2023visual} to extract dense spatial feature $\mathbf{Z}_v^{\text{s}}$ from the penultimate layer of $f_v$, which are then mapped into the LLM's embedding space via the projector, yielding projected visual tokens $\mathbf{P}_v = \mathcal{P}(\mathbf{Z}_v^{\text{s}}; \Phi_\mathcal{P})$. The LLM subsequently processes the concatenated sequence of visual tokens $\mathbf{P}_v$ and the embedded instruction tokens $x_{ins}$ to generate the response. Following standard paradigms, we keep the vision encoder parameters $\Phi_v$ frozen to preserve foundational visual representations. The trainable parameters $\Theta = \{\Phi_\mathcal{P}, \Delta\Phi_l\}$ typically include the projector weights $\Phi_\mathcal{P}$ and a set of Parameter-Efficient Fine-Tuning (PEFT) weights $\Delta\Phi_l$ (e.g., LoRA \cite{hu2022lora}) added to the LLM. The model is optimized via an autoregressive next-token prediction loss:
\begin{equation}
\mathcal{L}_{MLLM} = - \sum_{k=1}^{N_{ans}} \log p(x_{ans, k} \mid \mathbf{P}_v, x_{ins}, x_{ans, <k}; \Theta),
\end{equation}
where $N_{ans}$ is the sequence length of the answer. Upon encountering task $t$, the objective is to minimize the empirical risk over $\mathcal{D}_t$ while ensuring that $\Theta$ retains the knowledge acquired from preceding tasks $\{1, \dots, t-1\}$.

\begin{figure*}[t]
    \centering
    \includegraphics[width=1.0\linewidth]{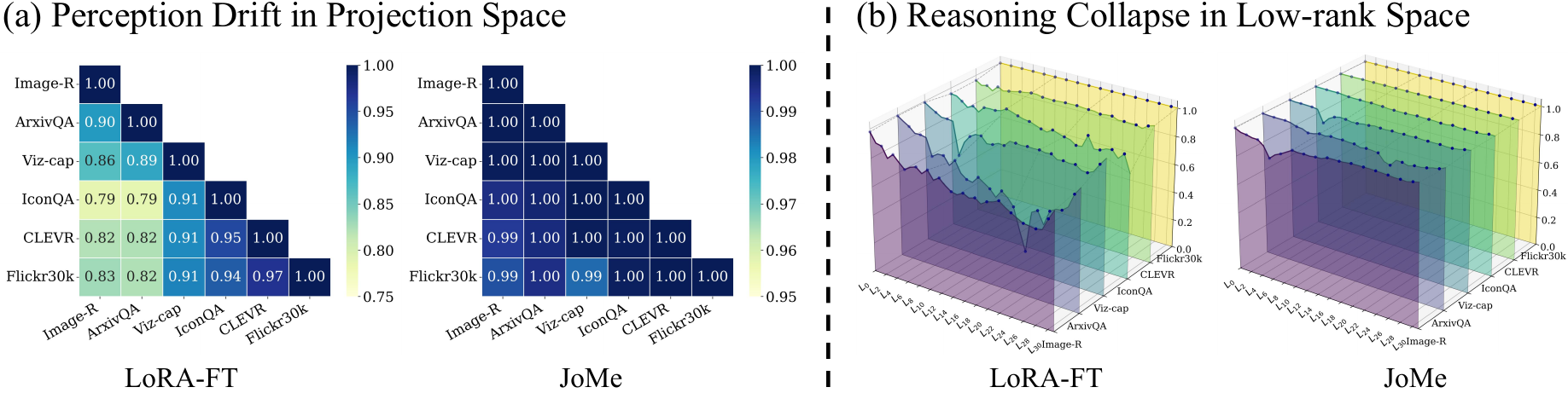}
    \vspace{-0.55cm}
    \caption{\textbf{Visualizing Perception Drift and Reasoning Collapse.} 
    (a) Cosine similarity of projector features across tasks. (b) Layer-wise feature similarity between oracle states and the final model. In both aspects, JoMe effectively mitigates internal representation drift.
    }
    \label{fig:feature_drift}
\end{figure*}

\subsection{Perception Drift and Reasoning Collapse}
\label{sec:preliminary_feature_drift}
To empirically motivate our approach, we investigate the internal dynamics of forgetting in MLLMs. As visualized in Figure~\ref{fig:feature_drift}, our analysis reveals a distinct dual-forgetting pattern that provides a mechanistic explanation for the performance drops diagnosed in Figure~\ref{fig:motivation_analysis}. Specifically, we examine the representation drift in both the projector and the LLM layers to identify the root causes of failure in sequential instruction tuning.
Figure~\ref{fig:feature_drift}(a) evaluates the stability of the cross-modal projection space. It measures the cosine similarity between the projector's output features in their oracle state (immediately after training on task $i$) and the features extracted after learning subsequent tasks $j > i$, using the exact samples from task $i$. In parallel, Figure~\ref{fig:feature_drift}(b) focuses on the internal layers of the LLM hierarchy, visualizing the similarity between the features of task $i$ at its oracle state and those extracted by the final model (after $n$ tasks). This analysis adds a depth dimension to observe how reasoning structures evolve across different model layers.

Our observations confirm two distinct failure modes that hinder successful continual learning. As shown in Figure~\ref{fig:feature_drift}(a), naive LoRA fine-tuning (LoRA-FT) suffers from a noticeable perception drift, where the similarity for initial tasks drops to 0.83 by the final stage. This drift disrupts the vision-language alignment, causing the projector to misinterpret visual inputs from earlier learning stages. In stark contrast, Figure~\ref{fig:feature_drift}(b) reveals a severe reasoning collapse within the LLM layers. For LoRA-FT, we observe a consistent and significant decline in feature similarity across all model layers. This degradation is particularly pronounced for the earliest trained tasks, reflecting the cumulative interference that compromises the model's reasoning stability over the task sequence. Conversely, our proposed JoMe demonstrates superior stability across both the cross-modal projection space and the hierarchical LLM layers. By maintaining feature consistency, our framework effectively overcomes both perception drift and reasoning collapse throughout the entire task sequence, ensuring the model's internal representations remain robust.

\begin{figure*}[t]
    \centering
    \includegraphics[width=1.0\linewidth]{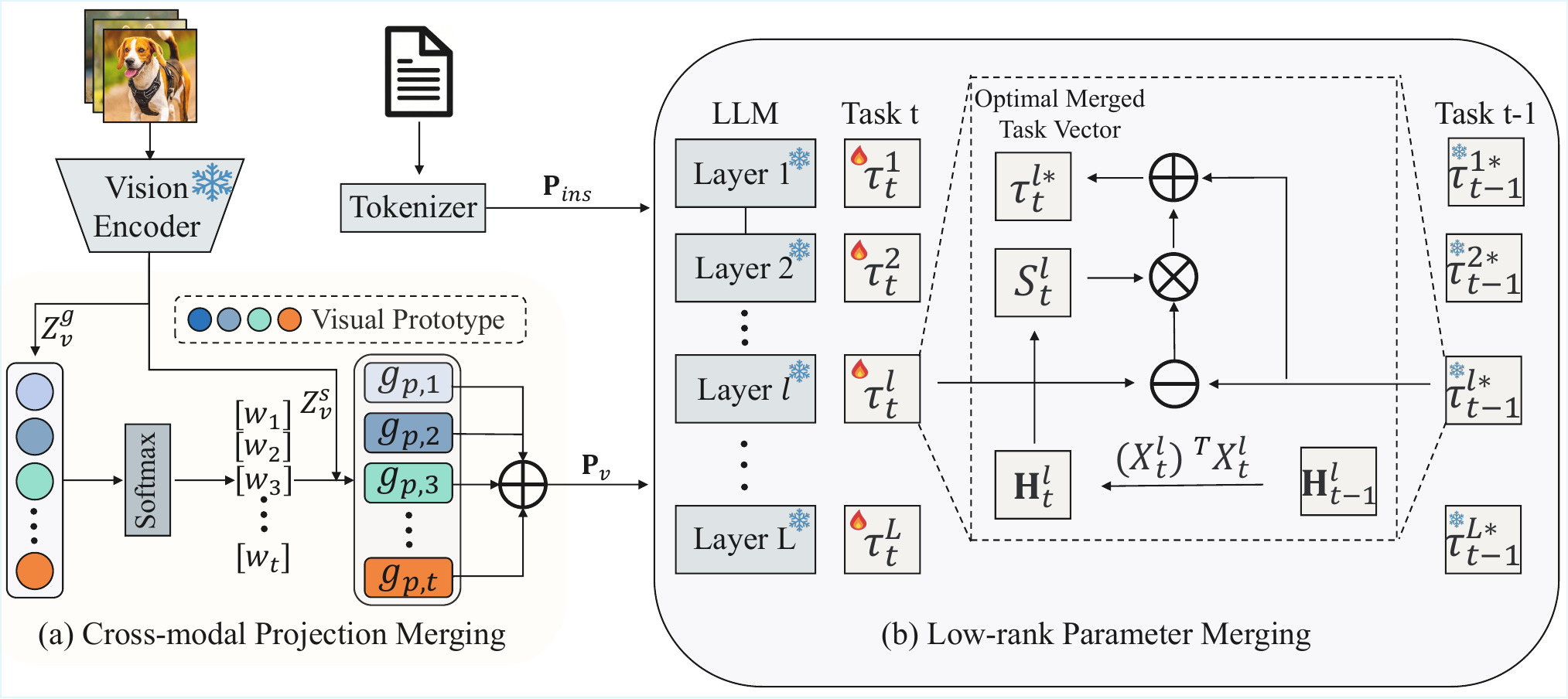}
    \vspace{-0.5cm}
    \caption{Overview of JoMe. The dual-track design explicitly decouples perception from reasoning.}
    \label{fig:method_architecture}
    \vspace{-0.4cm}
\end{figure*}

\section{Method}
\subsection{Decoupling Perception from Reasoning}
\label{sec:method_architecture}
To resolve the structural conflicts and divergent representation dynamics identified in Sec.~\ref{sec:preliminary_feature_drift}, we propose a dual-track architecture that explicitly decouples perception from reasoning. Our framework independently stabilizes these functional spaces through the Perceptual Track (Cross-modal Projection Merging, detailed in Sec.~\ref{sec:method_routing}) and the Reasoning Track (Low-rank Parameter Merging, detailed in Sec.~\ref{sec:method_recursive_fusion}). This component-aware design ensures that both perceptual alignment and reasoning stability are preserved throughout the task sequence, alleviating the dual-forgetting inherent in standard MLLM continual learning paradigms. Figure~\ref{fig:method_architecture} shows the overall design of our JoMe.

\subsection{Cross-modal Projection Merging}
\label{sec:method_routing}
To alleviate the perception drift, we introduce Cross-modal Projection Merging (CPM) and move away from the shared-projector bottleneck and instead treat the projection space as a collection of task-specific alignment layers. For each task $i \in \{1, \dots, t\}$, we maintain a task-specific lightweight projector $\mathcal{P}_{i}: \mathbb{R}^{d_v} \to \mathbb{R}^{d_l}$. Since these projectors are typically implemented as Multi-Layer Perceptrons (MLPs), the additional parameter overhead remains minimal compared to the LLM backbone, ensuring memory efficiency throughout the task sequence.
To enable task-agnostic inference, CPM employs a dynamic merging strategy that adaptively activates these projectors based on the visual context. During the task $t$, we compute a task-representative prototype $\mathbf{\mu}_t$ by aggregating visual features through the frozen vision encoder. Specifically, we utilize the 0-th index feature (i.e., the [CLS] token) from the last layer of $f_v$ as the global representation $\mathbf{Z}_v^{\text{g}}$ and compute prototype 
$\mathbf{\mu}_t = \frac{1}{N_t} \sum_{j=1}^{N_t} \mathbf{Z}_v^{\text{g}, j}$.
This prototype $\mathbf{\mu}_t$ serves as a stable anchor in the visual domain, capturing the specific distribution of task $t$ without requiring any textual metadata.

During inference, the model encounters a query image $x_v$ with an unknown task identity. We first extract its global semantic feature $\mathbf{Z}_v^{\text{g}}$ from the last layer and dense spatial feature $\mathbf{Z}_v^{\text{s}}$ from the penultimate layer of the vision encoder. We then measure the cosine similarity $s_i$ between the query's global feature and each stored prototype $\mathbf{\mu}_i$ for $i \in \{1, \dots, t\}$:
\begin{equation}
\label{eq:similarity}
s_i = \frac{(\mathbf{Z}_v^{\text{g}})^\top \mathbf{\mu}_i}{\| \mathbf{Z}_v^{\text{g}}\| \|\mathbf{\mu}_i\|}.
\end{equation}
These similarities are then normalized into merging weights $w_i$ using a temperature-scaled softmax operation. The final projected visual tokens $\mathbf{P}_v$ are computed as an adaptive merging of task-specific projectors acting on the spatial feature:
\begin{equation}
\label{eq:merging_projection}
w_i = \frac{\exp(s_i / \eta)}{\sum_{j=1}^t \exp(s_j / \eta)}, \quad \mathbf{P}_v = \sum_{i=1}^t w_i \mathcal{P}_{i}(\mathbf{Z}_v^{\text{s}}),
\end{equation}
where $\eta$ is a temperature hyperparameter, which is empirically set to $0.1$. This soft merging mechanism ensures robustness: for samples with ambiguous boundaries, the model can adaptively blend alignment knowledge from multiple relevant projectors. By isolating task-specific mappings and merging them only at test time, CPM effectively prevents the overwriting of prior alignments, ensuring the stability of the vision-language interface throughout the task sequence.

\subsection{Low-rank Parameter Merging}
\label{sec:method_recursive_fusion}
While CPM preserves perceptual alignments, the internal reasoning layers of the LLM still suffer from severe reasoning collapse. Here, we propose Low-Rank Parameter Merging (LPM), which recursively consolidates task-specific LoRA weights into a consolidated global task vector.

\subsubsection{Parameter-Efficient Task Vectors}
\label{sec:preliminary_task_vectors}
To continually incorporate the new LoRA weights in each layer of the LLM, we formulate our solution through the lens of model merging \cite{sun2025towards,ilharco2022editing}. Consider the $L$-layer pre-trained LLM $h_l$ with frozen weights $W_{\text{pre}} = \{W_{\text{pre}}^l\}_{l=1}^L$, where $W_{\text{pre}}^l \in \mathbb{R}^{d_{\text{out}} \times d_{\text{in}}}$. Following the parameterization $\Theta = \{\Phi_\mathcal{P}, \Delta\Phi_l\}$ defined in Sec.~\ref{sec:preliminary_cit}, we implement the reasoning track by embedding LoRA modules into the linear layers of the LLM.
For a specific layer $l$, the task-specific update for task $t$ is represented as a low-rank task vector $\tau_t^l = (B_t^l A_t^l)^\top$, where $B_t^l \in \mathbb{R}^{d_{\text{out}} \times r}$ and $A_t^l \in \mathbb{R}^{r \times d_{\text{in}}}$. Critically, following the LoRA \cite{hu2022lora} fine-tuning strategy used in most MLLMs \cite{liu2024improved}, we also employ a strictly bias-free LoRA architecture to ensure linear additivity. Under our PEFT regime, since the LLM backbone $W_{\text{pre}}$ remains frozen, the task vector for the $l$-th layer is the transpose of the learned low-rank increment:
\begin{equation}
\tau_t^l = (W_t^l - W_{\text{pre}}^l)^\top
         = (\Delta W_t^l)^\top
         = (B_t^l A_t^l)^\top.
\end{equation}
This design allows us to reframe the challenge of continual learning as a layer-wise model merging problem in the parameter subspace: our objective is to continuously integrate newly acquired task vectors $\tau_t^l$ into a consolidated global task vector $\tau^{l*}$, thereby achieving knowledge consolidation while minimizing interference with preceding tasks.

\subsubsection{Resolving Layer-wise Task Vector Conflicts}
While the task vector formulation allows for layer-wise merging, a straightforward arithmetic summation (e.g., $\tau^l = \sum_{i=1}^t \tau_i^l$ \cite{ilharco2022editing}) ignores the fact that different task vectors often occupy interfering directions in the weight space. This conflict manifests as negative transfer \cite{yadav2023ties,sun2025task}, where the performance degradation on task $i$ is formally defined as the loss gap between the merged models with layer weights $W_{\text{merge}}^l = W_{\text{pre}}^l + (\tau^l)^\top$ and $W_i^l = W_{\text{pre}}^l + (\tau_i^l)^\top$:
\begin{equation}
\begin{aligned}
\Delta \mathcal{L}_i
&= \mathbb{E}_{(x_v,x_{ins},x_{ans})\sim\mathcal{D}_i}
   \bigl[\mathcal{L}(W_{\text{merge}}; x_v,x_{ins},x_{ans})\bigr] \\
&\quad - \mathbb{E}_{(x_v,x_{ins},x_{ans})\sim\mathcal{D}_i}
   \bigl[\mathcal{L}(W_i; x_v,x_{ins},x_{ans})\bigr].
\end{aligned}
\end{equation}
Our goal is to find the optimal task vectors $\tau^{*}$ that jointly minimize this gap $\Delta \mathcal{L}_i$ across all encountered tasks.
Directly analyzing this global degradation is intractable due to the complex hierarchical structure of MLLMs. However, with a fixed projector and under the Lipschitz assumptions of \cite{sun2025cat}, the empirical loss gap admits a layer-wise bound:
\begin{equation}
|\Delta \mathcal{L}_i| \le \beta \sum_{l=1}^L \left( \prod_{m=l+1}^L \gamma_m \right) \|\Delta f_i^l\|_F,
\end{equation}
where $\beta$ and $\gamma_m$ are the Lipschitz constants of the loss function and the $m$-th layer of the merged model, respectively, and $\Delta f_i^l = f_i^l(W_{\text{merge}}^l) - f_i^l(W_i^l)$ denotes the feature drift on the task-specific model's reference activations for the same examples. This bound motivates minimizing $\|\Delta f_i^l\|_F$ layer by layer.
Leveraging the strictly bias-free LoRA architecture established in Sec.~\ref{sec:preliminary_task_vectors}, the feature transformation at any linear layer $l$ simplifies to a pure matrix multiplication: $f_i^l(W^l) = \mathbf{X}_i^l (W^l)^\top$, where $\mathbf{X}_i^l$ denotes the input feature matrix for task $i$. Substituting this into the drift definition, the feature drift at layer $l$ is expressed as:
\begin{equation}
\Delta f_i^l = \mathbf{X}_i^l \bigl((W_{\text{merge}}^l)^\top - (W_i^l)^\top\bigr) = \mathbf{X}_i^l(\tau^l - \tau_i^l).
\end{equation}
To maximize knowledge retention up to task $t$, the goal is to find an optimal merged vector $\tau_t^{l*}$ that incorporate the historical task vectors $\{\tau_1^l, \dots, \tau_{t-1}^l\}$ with the current vector $\tau_t^l$ while minimizing the cumulative feature drift. Using the input features $\{\mathbf{X}_i^l\}_{i=1}^t$ and task vectors $\{\tau_i^l\}_{i=1}^t$, we define the following layer-wise least-squares objective:
\begin{equation}
\label{eq:lpm_optimization_goal}
\tau_t^{l*} = \arg \min_{\tau^l} \sum_{i=1}^{t} \| \mathbf{X}_i^l (\tau^l - \tau_i^l) \|_F^2.
\end{equation}
By setting the gradient of Eq.~\ref{eq:lpm_optimization_goal} with respect to $\tau^l$ to zero, we obtain the batch least squares solution, provided the cumulative Gram matrix is invertible:
\begin{equation}
\label{eq:batch_solution_final}
\tau_t^{l*} = \left( \sum_{i=1}^t \mathbf{X}_i^{l\top} \mathbf{X}_i^l \right)^{-1} \sum_{i=1}^t \mathbf{X}_i^{l\top} \mathbf{X}_i^l \tau_i^l.
\end{equation}
Directly computing Eq.~\ref{eq:batch_solution_final} necessitates joint access to all historical features $\{\mathbf{X}_i^l\}_{i=1}^t$ and task vectors $\{\tau_i^l\}_{i=1}^t$, which is structurally intractable in exemplar-free continual learning.

\subsubsection{Exemplar-free Recursive Parameter Merging}
\label{sec:Recursive_Merging}
To satisfy the exemplar-free constraint, we reformulate the batch solution in Eq.~\ref{eq:batch_solution_final} into a recursive update rule that allows for incremental parameter merging.
We first represent the historical data distribution through a sufficient statistic, namely the cumulative feature covariance:
\begin{Definition}[Cumulative Feature Covariance]
For each layer $l$, the cumulative feature covariance matrix $\mathbf{H}_t^l \in \mathbb{R}^{d \times d}$ is defined as the unnormalized second-order moment of all input features encountered up to task $t$:
\begin{equation}
\mathbf{H}_t^l = \sum_{i=1}^t \mathbf{X}_i^{l\top} \mathbf{X}_i^l = \mathbf{H}_{t-1}^l + \mathbf{X}_t^{l\top} \mathbf{X}_t^l.
\end{equation}
The matrix $\mathbf{H}_t^l$ captures the relative importance of different parameter directions, enabling the model to remember historical feature structures without storing raw samples.
\end{Definition}

\begin{table*}[t]
    \caption{Performance on the UCIT benchmark with LLaVA-1.5-7B and InternVL-Chat-7B. Bold marks the best result for each metric within each backbone. Red values show absolute gains in FAA and CAA over the strongest baseline for the corresponding backbone and metric.}
    \vspace{-0.3cm}
    \label{tab:main_results}
    \centering
    \resizebox{\textwidth}{!}{%
    \begin{tabular}{c l l c c c c c c | c c c}
        \toprule
        & Method & Venue & Image-R & ArxivQA & Viz-cap & IconQA & CLEVR & Flickr30k & FAA ($\uparrow$) & CAA ($\uparrow$) & FFM ($\downarrow$) \\
        \midrule

        \multirow{11}{*}{\rotatebox{90}{\textbf{LLaVA-1.5-7B}}}
        & Zero-shot & -- & 16.27 & 57.73 & 38.39 & 19.20 & 20.63 & 41.88 & -- & -- & -- \\
        \cmidrule{2-12}
        & LoRA-FT \cite{hu2022lora} & ICLR'22 & 65.50 & 76.23 & 44.41 & 66.37 & 54.87 & \textbf{57.80} & 60.86 & 67.61 & 17.44 \\
        & O-LoRA \cite{wang2023orthogonal} & EMNLP'23 & 83.40 & \textbf{94.20} & 41.51 & 58.50 & 56.80 & 53.37 & 64.63 & 75.98 & 8.03 \\
        & MoELoRA \cite{chen2024coin} & NeurIPS'24 & 68.57 & 77.73 & 44.26 & 65.43 & 45.20 & 57.31 & 59.75 & 75.57 & 17.54 \\
        & ModalPrompt \cite{zeng2025modalprompt} & EMNLP'25 & 51.07 & 87.27 & 48.11 & 39.23 & 46.57 & 42.93 & 52.53 & 57.67 & \textbf{0.15} \\
        & CL-MoE \cite{huai2025cl} & CVPR'25 & 67.80 & 73.47 & 44.58 & 69.37 & 47.80 & 57.40 & 60.07 & 74.89 & 18.25 \\
        & HiDe-LLaVA \cite{guo2025hide} & ACL'25 & 86.93 & 91.37 & 48.68 & 67.80 & 47.87 & 52.95 & 65.93 & 79.03 & 5.58 \\
        & RegLoRA \cite{chen2025sefe} & ICML'25 & 83.53 & 65.77 & 45.72 & 75.23 & 63.03 & 57.74 & 65.17 & 76.64 & 12.80 \\
        & KeepLoRA \cite{luo2026keeplora} & ICLR'26 & 82.37 & 87.10 & 46.27 & 74.33 & 67.57 & 57.13 & 69.13 & 81.21 & 9.19 \\
        \cmidrule{2-12}
        & \textbf{JoMe}$^{*}$ & \textbf{Ours} & \textbf{91.30} & 92.63 & 59.54 & 79.63 & \textbf{69.37} & 56.96
        & \textbf{74.91} \textcolor{red}{\scriptsize (+5.78)}
        & \textbf{83.89} \textcolor{red}{\scriptsize (+2.68)} & 0.92 \\
        & \textbf{JoMe} & \textbf{Ours} & 90.87 & 91.97 & \textbf{60.48} & \textbf{80.20} & 67.37 & 56.09
        & 74.50 \textcolor{red}{\scriptsize (+5.37)}
        & 83.42 \textcolor{red}{\scriptsize (+2.21)} & 0.37 \\
        \midrule

        \multirow{11}{*}{\rotatebox{90}{\textbf{InternVL-Chat-7B}}}
        & Zero-shot & -- & 21.10 & 63.20 & 40.59 & 24.70 & 21.20 & 44.67 & -- & -- & -- \\
        \cmidrule{2-12}
        & LoRA-FT \cite{hu2022lora} & ICLR'22 & 78.87 & 76.23 & 45.36 & 66.43 & 70.13 & \textbf{58.45} & 65.91 & 77.89 & 10.12 \\
        & O-LoRA \cite{wang2023orthogonal} & EMNLP'23 & 86.00 & \textbf{94.23} & 39.68 & 63.83 & 46.97 & 45.16 & 62.65 & 79.06 & 6.97 \\
        & MoELoRA \cite{chen2024coin} & NeurIPS'24 & 70.50 & 77.53 & 44.98 & 68.17 & 69.57 & 57.80 & 64.76 & 77.59 & 10.40 \\
        & ModalPrompt \cite{zeng2025modalprompt} & EMNLP'25 & 57.10 & 63.90 & 42.11 & 37.40 & 47.50 & 45.76 & 48.96 & 53.52 & \textbf{0.02} \\
        & CL-MoE \cite{huai2025cl} & CVPR'25 & 76.80 & 78.43 & 44.72 & 69.07 & \textbf{73.07} & 58.43 & 66.75 & 79.68 & 9.24 \\
        & HiDe-LLaVA \cite{guo2025hide} & ACL'25 & 89.33 & 91.07 & 49.55 & 67.80 & 63.63 & 55.62 & 69.50 & 81.33 & 2.03 \\
        & RegLoRA \cite{chen2025sefe} & ICML'25 & 86.77 & 78.13 & 45.90 & 70.67 & 69.80 & 58.18 & 68.24 & 79.87 & 6.18 \\
        & KeepLoRA \cite{luo2026keeplora} & ICLR'26 & 84.37 & 91.10 & 46.15 & \textbf{77.33} & 67.77 & 57.05 & 70.63 & 81.73 & 7.97 \\
        \cmidrule{2-12}
        & \textbf{JoMe}$^{*}$ & \textbf{Ours} & \textbf{94.60} & 81.80 & 57.86 & 70.97 & 69.20 & 56.62
        & 71.84 \textcolor{red}{\scriptsize (+1.21)}
        & 83.10 \textcolor{red}{\scriptsize (+1.37)} & 3.00 \\
        & \textbf{JoMe} & \textbf{Ours} & \textbf{94.60} & 83.60 & \textbf{59.86} & 71.93 & 67.63 & 56.48
        & \textbf{72.35} \textcolor{red}{\scriptsize (+1.72)}
        & \textbf{83.35} \textcolor{red}{\scriptsize (+1.62)} & 1.93 \\
        \bottomrule
    \end{tabular}%
    }
    \vspace{-0.6cm}
\end{table*}

By maintaining the matrix $\mathbf{H}_t^l$, the optimal solution can be updated recursively as summarized in the following theorem:
\begin{theorem}[Recursive Least Squares (RLS) Solution]
\label{theorem:rls_merging}
The optimal merged task vector $\tau_t^{l*}$ that minimizes the cumulative feature drift up to task $t$ can be obtained recursively:
\begin{equation}
\label{eq:rls_update}
\tau_t^{l*} = \tau_{t-1}^{l*} + \mathbf{S}_t^l (\tau_t^l - \tau_{t-1}^{l*}), \quad \text{where } \mathbf{S}_t^l =  (\mathbf{H}_t^l)^{-1} (\mathbf{X}_t^{l\top} \mathbf{X}_t^l).
\end{equation}
Here, $\mathbf{S}_t^l$ denotes the gain matrix that adaptively modulates the task residual $(\tau_t^l - \tau_{t-1}^{l*})$. This update rule achieves exact mathematical equivalence to the batch least squares solution in Eq.~\ref{eq:batch_solution_final} without requiring raw historical data.
\end{theorem}
\begin{proof}
See Appendix \ref{sec:appendix_rls} for the detailed derivation.
\end{proof}

Unlike simple averaging, this recursive solution offers a clear advantage by using feature strength to prioritize important updates. As explained via Singular Value Decomposition (SVD)~\cite{SVD} in Appendix \ref{sec:appendix_mechanism}, this method naturally puts updates where they are most useful and removes parts that cause conflicts to prevent critical information from being lost even when different tasks update the same parts of the model. During testing, the merged update $\tau_t^{l*}$ is added into the frozen weights $W_{\text{pre}}^l$ using a scaling factor $\lambda$:
\begin{equation}
W_{\text{final}}^l = W_{\text{pre}}^l + \lambda (\tau_t^{l*})^\top
\end{equation}
Adjusting $\lambda$ follows common practice in model merging~\cite{yadav2023ties,du2024parameter,sun2025towards} to balance new task knowledge with the pre-trained model for better performance across downstream tasks.

Additionally, to alleviate the storage burden, we avoid storing the full task vector $\tau_t^{l*}$ by merging it directly into the pre-trained weights. We recover the task vector only when it is needed by calculating the difference between the model and the pre-trained weights $W_{\text{pre}}^l$. Our method then focuses on managing the cumulative feature covariance $\mathbf{H}_t^l$ through two versions: \textbf{JoMe} and \textbf{JoMe$^*$}. In \textbf{JoMe}, we maintain the full $d \times d$ matrix $\mathbf{H}_t^l$. To handle the matrix memory cost, \textbf{JoMe$^*$} uses SVD to compress $\mathbf{H}_t^l$. We keep only the top-$r$ components where the sum of their squared singular values represents a sufficiently large portion of the total feature strength. Specifically, we choose the smallest $r$ such that the ratio meets an energy threshold $\gamma$ ($\frac{\sum_{i=1}^r \sigma_i^2}{\sum_{i=1}^d \sigma_i^2} \ge \gamma$), yielding the low-rank approximation $\mathbf{H}_t^l \approx \mathbf{U}_r^l \boldsymbol{\Sigma}_r^l \mathbf{V}_r^{l\top}$. This approach allows the model to remember historical feature structures using much less storage space while keeping the most useful information.

In summary, the dual-track design of JoMe preserves both perceptual and structural knowledge across tasks and eliminates the need for expensive additional GPU training. The pseudo code is included in Algorithm~\ref{alg:lpm_cpm_full}.

\begin{table*}[t]
    \caption{Performance on the MLLM-DCL benchmark with LLaVA-1.5-7B and InternVL-Chat-7B. Bold marks the best result for each metric within each backbone. Red values show gains in FAA and CAA over the strongest baseline for the corresponding backbone and metric.}
    \label{tab:mllm_dcl_results}
    \vspace{-0.3cm}
    \centering
    \resizebox{0.86\textwidth}{!}{%
    \begin{tabular}{c l l c c c c c | c c c}
        \toprule
        & Method & Venue & RS & Med & AD & Sci & Fin & FAA ($\uparrow$) & CAA ($\uparrow$) & FFM ($\downarrow$) \\
        \midrule

        \multirow{11}{*}{\rotatebox{90}{\textbf{LLaVA-1.5-7B}}}
        & Zero-shot & -- & 32.29 & 28.28 & 15.59 & 35.55 & 62.56 & -- & -- & -- \\
        \cmidrule{2-11}
        & LoRA-FT \cite{hu2022lora} & ICLR'22 & 74.13 & 47.63 & 33.53 & 41.42 & 89.66 & 57.27 & 63.23 & 11.43 \\
        & O-LoRA \cite{wang2023orthogonal} & EMNLP'23 & 74.79 & 39.78 & 30.68 & 39.11 & 83.34 & 53.54 & 59.34 & 7.03 \\
        & MoELoRA \cite{chen2024coin} & NeurIPS'24 & 76.96 & 47.09 & 31.19 & 41.96 & 89.74 & 57.39 & 63.11 & 11.65 \\
        & ModalPrompt \cite{zeng2025modalprompt} & EMNLP'25 & 47.92 & 29.10 & 13.35 & 32.77 & 89.03 & 42.43 & 49.75 & 30.52 \\
        & CL-MoE \cite{huai2025cl} & CVPR'25 & 73.05 & 48.00 & 30.27 & 41.48 & \textbf{90.11} & 56.58 & 62.81 & 12.74 \\
        & HiDe-LLaVA \cite{guo2025hide} & ACL'25 & 70.38 & 38.58 & 27.60 & 35.11 & 60.90 & 46.51 & 58.39 & 11.81 \\
        & RegLoRA \cite{chen2025sefe} & ICML'25 & 76.38 & 47.78 & 40.16 & 45.64 & 89.50 & 59.89 & 65.13 & 7.98 \\
        & KeepLoRA \cite{luo2026keeplora} & ICLR'26 & 77.20 & 52.79 & 43.00 & 46.18 & 89.95 & 61.82 & 66.26 & 6.29 \\
        \cmidrule{2-11}
        & \textbf{JoMe}$^{*}$ & \textbf{Ours} & 79.58 & \textbf{59.02} & \textbf{49.56} & \textbf{49.53} & 89.62
        & \textbf{65.46} \textcolor{red}{\scriptsize (+3.64)}
        & \textbf{66.82} \textcolor{red}{\scriptsize (+0.56)} & -1.43 \\
        & \textbf{JoMe} & \textbf{Ours} & \textbf{79.72} & 58.85 & 47.63 & 49.39 & 89.43
        & 65.00 \textcolor{red}{\scriptsize (+3.18)}
        & 66.43 \textcolor{red}{\scriptsize (+0.17)} & \textbf{-1.61} \\
        \midrule

        \multirow{11}{*}{\rotatebox{90}{\textbf{InternVL-Chat-7B}}}
        & Zero-shot & -- & 31.16 & 29.81 & 14.06 & 33.93 & 64.32 & -- & -- & -- \\
        \cmidrule{2-11}
        & LoRA-FT \cite{hu2022lora} & ICLR'22 & 77.39 & 53.32 & 35.12 & 43.50 & 91.05 & 60.08 & 67.35 & 11.58 \\
        & O-LoRA \cite{wang2023orthogonal} & EMNLP'23 & 76.93 & 40.32 & 29.55 & 34.61 & 72.39 & 50.76 & 63.14 & 12.46 \\
        & MoELoRA \cite{chen2024coin} & NeurIPS'24 & 74.41 & 53.89 & 39.67 & 43.46 & \textbf{91.27} & 60.54 & 67.35 & 10.24 \\
        & ModalPrompt \cite{zeng2025modalprompt} & EMNLP'25 & 58.19 & 28.76 & 13.77 & 34.41 & 65.32 & 40.09 & 41.91 & 0.07 \\
        & CL-MoE \cite{huai2025cl} & CVPR'25 & 68.10 & 47.26 & 30.92 & 40.55 & 89.40 & 55.25 & 63.20 & 15.01 \\
        & HiDe-LLaVA \cite{guo2025hide} & ACL'25 & 77.73 & 58.10 & 35.87 & \textbf{45.55} & 90.14 & 61.48 & 66.77 & 5.81 \\
        & RegLoRA \cite{chen2025sefe} & ICML'25 & 78.57 & 55.09 & 45.27 & 44.88 & 90.29 & 62.82 & 67.85 & 8.33 \\
        & KeepLoRA \cite{luo2026keeplora} & ICLR'26 & 79.55 & 57.30 & 46.27 & 42.88 & 90.33 & 63.27 & 67.46 & 5.54 \\
        \cmidrule{2-11}
        & \textbf{JoMe}$^{*}$ & \textbf{Ours} & \textbf{81.25} & \textbf{65.44} & \textbf{50.84} & 45.43 & 90.55
        & \textbf{66.70} \textcolor{red}{\scriptsize (+3.43)}
        & \textbf{69.18} \textcolor{red}{\scriptsize (+1.33)} & -0.26 \\
        & \textbf{JoMe} & \textbf{Ours} & \textbf{81.25} & 65.13 & 50.28 & 44.99 & 90.55
        & 66.44 \textcolor{red}{\scriptsize (+3.17)}
        & 68.90 \textcolor{red}{\scriptsize (+1.05)} & \textbf{-0.56} \\
        \bottomrule
    \end{tabular}%
    }
    \vspace{-0.6cm}
\end{table*}

\section{Experiment}
\subsection{Experiment Setup}
\textbf{Datasets.} We evaluate our framework on two MCIT benchmarks: UCIT \cite{guo2025hide} and MLLM-DCL \cite{zhao2025mllm}. UCIT benchmark rigorously assesses true learning capability without pre-training data contamination. Based on LLaVA's \cite{liu2023visual} zero-shot performance, it comprises six sequential tasks strictly uncorrelated with standard MLLM corpora: ArxivQA \cite{li2024multimodal}, CLEVR-Math \cite{lindstrom2022clevr}, IconQA \cite{lu2021iconqa}, ImageNet-R \cite{hendrycks2021many}, VizWiz-caption (Viz-cap) \cite{gurari2018vizwiz}, and Flickr30k \cite{plummer2015flickr30k}. Conversely, MLLM-DCL benchmark evaluates the sequential acquisition of specialized knowledge across five distinct professional domains: remote sensing (RSVQA \cite{lobry2020rsvqa}), medicine (PathVQA \cite{he2020pathvqa}), autonomous driving (DriveLM \cite{sima2024drivelm}), science (AI2D \cite{kembhavi2016diagram}, SciVerse \cite{guo2025sciverse}, MapQA \cite{chang2022mapqa}, TQA \cite{kembhavi2017you}), and finance (StockQA \cite{wang2023finvis}).

\paragraph{Evaluation Metrics:}
    We utilize three widely-recognized evaluation metrics for continual learning - Final Average Accuracy (FAA), Cumulative Average Accuracy (CAA), and Final Forgetting Measure (FFM) as detailed in~\cite{wang2024comprehensive}. We define the accuracy on the task $\mathcal{T}^j$ after learning the task $\mathcal{T}^i$ as $A_{ij}$. The average accuracy after learning task $\mathcal{T}^i$ is denoted as $AA_{i} = \frac{1}{i} \sum_{j=1}^i A_{i j}$. Upon completing all $n$ tasks, we report $\mathrm{FAA} = AA_{n}$, $\mathrm{CAA} = \frac{1}{n} \sum_{i=1}^n AA_{i}$, and $\mathrm{FFM} = \frac{1}{n-1} \sum_{j=1}^{n-1} (A_{jj} - A_{nj})$.
The FAA is a critical metric highlighting performance discrepancies between CL methods and joint learning. The CAA provides a comprehensive view of overall historical performance, and the FFM quantifies the model's capability to mitigate forgetting.


\textbf{Implementation Details.} All algorithms are implemented using PyTorch~\cite{pytorch} and the MCITlib toolbox~\cite{guo2025mcitlib}, and are executed on an NVIDIA RTX A100 GPU cluster. We employ LLaVA-1.5-7B~\cite{liu2023visual} and InternVL-Chat-7B~\cite{chen2024internvl} as our foundational backbones, utilizing CLIP-L/14-336~\cite{radford2021learning} for visual feature extraction. 
For parameter-efficient fine-tuning, we inject LoRA modules with a rank of 16.
We set the scaling factor $\lambda$ to 3 and the energy threshold $\gamma$ in JoMe$^*$ to 0.999. See Appendix \ref{sup_implementation} for more details.

\textbf{Baseline.} We compare our JoMe with recent MCIT methods, including LoRA-FT \cite{hu2022lora}, O-LoRA \cite{wang2023orthogonal}, MoELoRA \cite{chen2024coin}, ModalPrompt \cite{zeng2025modalprompt}, CL-MoE \cite{huai2025cl}, HiDe-LLaVA \cite{guo2025hide}, RegLoRA \cite{chen2025sefe}, and KeepLoRA~\cite{luo2026keeplora}. The training process follows the replay-free continual learning setting \cite{chen2025sefe}.

\subsection{Main Results}
\label{sec:comparison_results}
Table \ref{tab:main_results} summarizes the core performance of JoMe on the UCIT \cite{guo2025hide} benchmark. Across both LLaVA-1.5-7B and InternVL-Chat-7B, JoMe consistently sets a new state-of-the-art (SOTA), resolving the core dilemma by achieving an unprecedented balance between plasticity and stability. Concretely, JoMe outperforms the prior strongest baseline KeepLoRA \cite{luo2026keeplora} by significant margins: +5.37\% in FAA on LLaVA-1.5-7B, with FFM reduced to an ultra-low 0.37. This superiority generalizes seamlessly to InternVL-Chat-7B, with a +1.72\% FAA gain over KeepLoRA \cite{luo2026keeplora} and lower forgetting, validating robust cross-architecture generalization. Our memory-efficient variant JoMe$^*$, which leverages SVD to drastically reduce storage overhead, fully retains the competitive performance of JoMe. 

Table~\ref{tab:mllm_dcl_results} compares JoMe with the baselines on MLLM-DCL~\cite{zhao2025mllm}, which evaluates continual instruction tuning across specialized domains with large cross-task distribution shifts.
JoMe outperforms the evaluated baselines in FAA and FFM on both backbones. On LLaVA-1.5-7B, JoMe improves FAA by 3.18\% over KeepLoRA~\cite{luo2026keeplora}, the strongest baseline. It achieves an FFM of -1.61\%, the lowest FFM in this comparison in the baseline comparison on this backbone, indicating positive backward transfer on average. On InternVL-Chat-7B, JoMe improves FAA over KeepLoRA by 3.17\% and achieves an FFM of -0.56\%.
, compared with 8.33\% for RegLoRA and 5.81\% for HiDe-LLaVA. LoRA-FT and CL-MoE incur FFM above 10\%, while ModalPrompt preserves stability but trails the best FAA by over 20\%. Meanwhile, the memory-efficient JoMe$^*$ slightly improves FAA further, reaching 65.46\% on LLaVA-1.5-7B and 66.70\% on InternVL-Chat-7B.

\subsection{Ablation Study}
\label{sec:ablation_study}

Table~\ref{tab:ablation_study} ablates JoMe to validate our dual-track architecture. The naive baseline suffers from a severe $17.44\%$ FFM. Introducing CPM alone reduces FFM to $10.45\%$ and improves FAA to $68.71\%$, mitigating perception drift. Conversely, applying LPM alone drops FFM to $5.85\%$, resolving internal reasoning conflicts. Combining both modules in JoMe yields the optimal FAA ($74.50\%$) and virtually eliminates forgetting, with highly consistent results on InternVL-Chat-7B. 
Further analyses on the synergy between perceptual and reasoning merging, the performance gap against theoretical bounds, and CPM's modality selection are detailed in Appendix \ref{sec:appendix_lpm_bound} and \ref{sec:appendix_cpm_modality}, respectively.

\begin{table*}[t]
\vspace{-0.3cm}
    \caption{Ablation study of different components in JoMe on the UCIT benchmark.}
    \vspace{-0.3cm}
    \label{tab:ablation_study}
    \centering
    \renewcommand{\arraystretch}{0.78}
    \setlength{\aboverulesep}{0.3ex}
    \setlength{\belowrulesep}{0.3ex}
    \resizebox{0.70\textwidth}{!}{
    \begin{tabular}{@{} c c c c c c c c @{}}
        \toprule
        \multirow{2}{*}{\textbf{CPM}} & \multirow{2}{*}{\textbf{LPM}} & \multicolumn{3}{c}{\textbf{LLaVA-1.5-7B}} & \multicolumn{3}{c}{\textbf{InternVL-Chat-7B}} \\
        \cmidrule(lr){3-5} \cmidrule(lr){6-8}
        & & FAA ($\uparrow$) & CAA ($\uparrow$) & FFM ($\downarrow$) & FAA ($\uparrow$) & CAA ($\uparrow$) & FFM ($\downarrow$) \\
        \midrule
        & & 60.86 & 67.61 & 17.44 & 65.91 & 77.89 & 10.12 \\
        \checkmark & & 68.71 & 75.65 & 10.45 & 68.20 & 80.14 & 8.22 \\
        & \checkmark & 70.87 & 81.93 & 5.85 & 69.99 & 81.85 & 3.71 \\
        \checkmark & \checkmark & \textbf{74.50} & \textbf{83.42} &\textbf{ 0.37} & \textbf{72.35} & \textbf{83.35} & \textbf{1.93} \\
        \bottomrule
    \end{tabular}
    }
    \vspace{-0.3cm}
\end{table*}

\subsection{Plug-and-Play Performance}
\label{sec:plug}

\begin{table*}[t]
    \caption{Comparison of different methods w/ and w/o the CPM module on the UCIT benchmark.}
    \vspace{-0.3cm}
    \label{tab:cpm_comparison}
    \centering
    \renewcommand{\arraystretch}{0.72} 
    \setlength{\aboverulesep}{0.25ex}
    \setlength{\belowrulesep}{0.25ex}
    \setlength{\tabcolsep}{4pt} 
    \resizebox{0.70\textwidth}{!}{
    \begin{tabular}{@{} l c c c c c c @{}}
        \toprule
        \multirow{2}{*}{\textbf{Method}} & \multicolumn{3}{c}{\textbf{w/o CPM}} & \multicolumn{3}{c}{\textbf{w/ CPM}} \\
        \cmidrule(lr){2-4} \cmidrule(lr){5-7}
        & FAA ($\uparrow$) & CAA ($\uparrow$) & FFM ($\downarrow$) & FAA ($\uparrow$) & CAA ($\uparrow$) & FFM ($\downarrow$) \\
        \midrule
        Task-specific LoRA   & 73.99 & 82.56 & 4.40 & 76.98 & 84.01 & 0.23 \\
        LoRA-FT     & 60.86 & 67.61 & 17.44 & 68.71 & 75.65 & 10.45 \\ 
        O-LoRA      & 64.63 & 75.98 & 8.03 & 68.53 & 80.58 & 2.78 \\
        HiDe-LLaVA        & 65.93 & 79.03 & 5.58 & 68.93 & 80.23 & 1.96 \\
        RegLoRA        & 65.17 & 76.64 & 12.80 & 69.04    &79.37     & 5.76
        
        \\
        JoMe (Ours) & 70.87 & 81.93 & 5.85 & 74.50 & 83.42 & 0.37 \\
        \bottomrule
    \end{tabular}
    }
    \vspace{-0.5cm}
\end{table*}

Table~\ref{tab:cpm_comparison} underscores CPM's versatility. In the w/o CPM setting, baselines are forced to utilize the projector from the final task, which triggers severe perception drift (e.g., LoRA-FT's $17.44\%$ FFM). 
With task-specific LoRA, CPM raises FAA from 73.99\% to 76.98\%, narrowing the gap to the task-projector oracle (77.66\%). CPM also reduces forgetting across the evaluated baselines, lowering HiDe-LLaVA's FFM from $5.58\%$ to $1.96\%$.



\subsection{Further Analysis}
\label{sec:Further_study}

\textbf{Sensitivity of the Scaling Factor $\lambda$.} 
Figure~\ref{fig:lambda_cropped} evaluates JoMe's performance across various $\lambda$ values, where dashed lines indicate the current SOTA performance. Our recursive merging neutralizes interfering parameter directions, which naturally attenuates the magnitude of the merged task vector $\tau_t^{l*}$. Consequently, a larger scaling factor is required to amplify the purified knowledge against the frozen backbone. For both LLaVA and InternVL, JoMe outperforms the SOTA over a broad range of $\lambda$ values. Notably, the FAA metric demonstrates exceptional robustness and maintains great performance once $\lambda$ exceeds $2.0$. This saturation around $\lambda \approx 3.0$ confirms that an appropriate scaling factor effectively compensates for RLS-induced attenuation without compromising stability. We set $\lambda = 3$ as the default. See Appendix~\ref{sensiety_gamma} for the hyperparameter $\gamma$ analysis.

\begin{figure}[t]
    \centering
    \includegraphics[width=0.85\columnwidth]{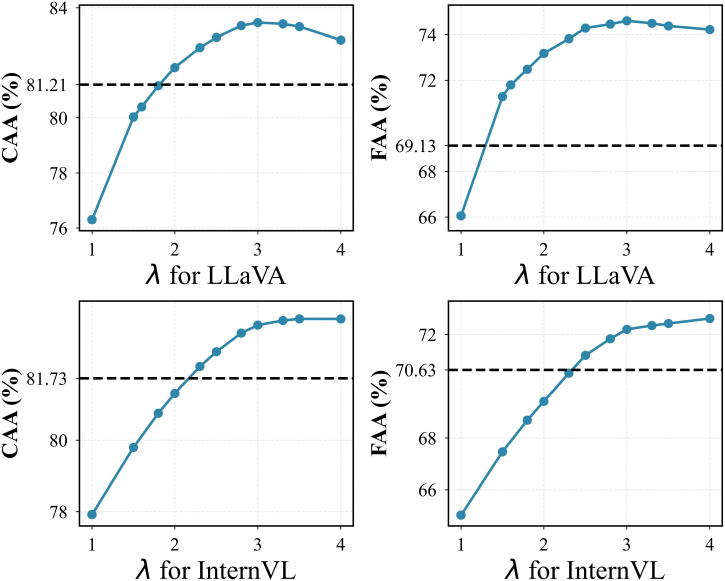}
    \caption{Sensitivity to $\lambda$. Dashed lines indicate the best baseline results for the corresponding metrics.}
    \label{fig:lambda_cropped}
    \vspace{-0.6cm}
\end{figure}

\section{Conclusion}
In this work, we identified a neglected dual-forgetting phenomenon in MCIT involving perception drift in the projection space and reasoning collapse in the parameter space. To resolve these conflicts, we presented JoMe, a training-free framework that decouples perception from reasoning through a dual-track merging design. 
CPM mitigates perceptual drift through prototype-guided projection merging, while LPM consolidates task-specific LoRA updates via recursive least squares.
Extensive evaluations on the UCIT and MLLM-DCL benchmarks demonstrate that JoMe consistently sets a new state-of-the-art across various MLLMs. By leveraging a training-free merging paradigm, JoMe eliminates additional gradient-based optimization beyond naive task tuning, providing a robust and lightweight solution for multimodal continual instruction tuning.

\section*{LLM Usage Statement} 
We used DeepSeek to assist with manuscript writing and to refine wording, correct grammar, and improve readability throughout the manuscript. We did not use generative AI tools for research ideation, study design, theoretical proofs, algorithm development, code implementation, data collection, or analysis of experimental results. We take full responsibility for the final content of the manuscript.
{
    \small
    \bibliographystyle{ieeenat_fullname}
    \bibliography{ref}
}
\clearpage
\maketitlesupplementary

\section{Derivation and Equivalence Proof of Theorem~\ref{theorem:rls_merging}}
\label{sec:appendix_rls}
This appendix provides the formal algebraic derivation of the Recursive Least Squares (RLS) merging rule introduced in Sec.~\ref{sec:Recursive_Merging}. We prove that the incremental update in Theorem~\ref{theorem:rls_merging} is mathematically equivalent to the batch least squares solution.

For a specific layer $l$, let $\mathbf{H}_t^l \in \mathbb{R}^{d \times d}$ be the cumulative feature covariance and momentum $\mathbf{Q}_t^l \in \mathbb{R}^{d_{\text{in}} \times d_{\text{out}}}$ be the cumulative weighted task vector. These sufficient statistics are defined as:
\begin{equation}
\label{eq:appendix_def}
\mathbf{H}_t^l = \sum_{i=1}^t \mathbf{X}_i^{l\top} \mathbf{X}_i^l, \quad \mathbf{Q}_t^l = \sum_{i=1}^t  \mathbf{X}_i^{l\top} \mathbf{X}_i^l \tau_i^l.
\end{equation}
As derived in Eq.~\ref{eq:batch_solution_final}, the optimal batch solution $\tau_t^{l*}$ that minimizes the cumulative feature drift is given by:
\begin{equation}
\label{eq:appendix_batch}
\tau_t^{l*} = (\mathbf{H}_t^l)^{-1} \mathbf{Q}_t^l.
\end{equation}

Upon encountering the $t$-th task, the statistics $\mathbf{H}_t^l$ and $\mathbf{Q}_t^l$ can be updated from their previous states at task $t-1$:
\begin{align}
\mathbf{H}_t^l &= \mathbf{H}_{t-1}^l + \mathbf{X}_t^{l\top} \mathbf{X}_t^l, 
\label{eq:recur_h} \\
\mathbf{Q}_t^l &= \mathbf{Q}_{t-1}^l + \mathbf{X}_t^{l\top} \mathbf{X}_t^l \tau_t^l. \label{eq:recur_q}
\end{align}
From Eq.~\ref{eq:appendix_batch}, we have the relationship $\mathbf{Q}_{t-1}^l = \mathbf{H}_{t-1}^l \tau_{t-1}^{l*}$.

To derive the update rule for $\tau_t^{l*}$, we substitute the recurrence \eqref{eq:recur_q} into the batch form \eqref{eq:appendix_batch}:
\begin{equation}
\begin{aligned}
\tau_t^{l*} &= (\mathbf{H}_t^l)^{-1} \mathbf{Q}_t^l = (\mathbf{H}_t^l)^{-1} \left( \mathbf{Q}_{t-1}^l + \mathbf{X}_t^{l\top} \mathbf{X}_t^l \tau_t^l \right) \\
&= (\mathbf{H}_t^l)^{-1} \left( \mathbf{H}_{t-1}^l \tau_{t-1}^{l*} + \mathbf{X}_t^{l\top} \mathbf{X}_t^l \tau_t^l \right).
\end{aligned}
\end{equation}
By rearranging Eq.~\ref{eq:recur_h} as $\mathbf{H}_{t-1}^l = \mathbf{H}_t^l - \mathbf{X}_t^{l\top} \mathbf{X}_t^l$, we can eliminate the historical covariance $\mathbf{H}_{t-1}^l$:
\begin{equation}
\begin{adjustbox}{max width=1.0\linewidth}
$\displaystyle
\begin{aligned}
\tau_t^{l*}
&= (\mathbf{H}_t^l)^{-1}
   \left[
      (\mathbf{H}_t^l - \mathbf{X}_t^{l\top}\mathbf{X}_t^l)
      \tau_{t-1}^{l*}
      + \mathbf{X}_t^{l\top}\mathbf{X}_t^l \tau_t^l
   \right] \\
&= (\mathbf{H}_t^l)^{-1}\mathbf{H}_t^l \tau_{t-1}^{l*}
   - (\mathbf{H}_t^l)^{-1}\mathbf{X}_t^{l\top}\mathbf{X}_t^l
     \tau_{t-1}^{l*}
   + (\mathbf{H}_t^l)^{-1}\mathbf{X}_t^{l\top}\mathbf{X}_t^l
     \tau_t^l \\
&= \tau_{t-1}^{l*}
   + (\mathbf{H}_t^l)^{-1}
     (\mathbf{X}_t^{l\top}\mathbf{X}_t^l)
     (\tau_t^l - \tau_{t-1}^{l*}).
\end{aligned}$
\end{adjustbox}
\end{equation}

\begin{algorithm}[t]
	\caption{The Proposed JoMe Method}
	\label{alg:lpm_cpm_full}
	\textbf{Input}: Pre-trained weights $\{W_{\text{pre}}^l, \Phi_v\}$, Task sequence $\{\mathcal{T}^t\}_{t=1}^n$, Scaling factor $\lambda$, Temperature $\eta$\\
	\textbf{Output}: Merged task vectors $\{\tau_n^{l*}\}_{l=1}^L$, Task-specific Projectors $\{\mathcal{P}_{i}\}_{i=1}^n$, and Prototypes $\{\mathbf{\mu}_i\}_{i=1}^n$
	\begin{algorithmic}[1] 
		\STATE \textbf{Training Phase:}
		\STATE Initialize cumulative covariance $\mathbf{H}_0^l = \mathbf{0}$ and global task vector $\tau_0^{l*} = \mathbf{0}$ for each layer $l \in \{1, \dots, L\}$
		\FOR {$t = 1$ \textbf{to} $n$}
		\STATE \textbf{1. Task-specific Training}: 
        \STATE Fine-tune LoRA weights $\tau_t^l$ and task-specific projector $\mathcal{P}_{t}$ on $\mathcal{T}^t$
		
		\STATE \textbf{2. Perceptual Track (Cross-modal Projection Merging)}: 
		\STATE \quad Extract [CLS] features $\mathbf{Z}_v^{\text{g}}$ from the last layer of $\Phi_v$
		\STATE \quad Compute task-representative prototype: $\mathbf{\mu}_t = \frac{1}{N_t} \sum_{j=1}^{N_t} \mathbf{Z}_v^{\text{g}, j}$ 
		\STATE \quad Store $\{\mathcal{P}_{t}, \mathbf{\mu}_t\}$ for test-time adaptive merging
		
		\STATE \textbf{3. Reasoning Track (Low-rank Parameter Merging)}:
		\FOR {each reasoning layer $l = 1$ \textbf{to} $L$}
		\STATE \quad Update cumulative feature covariance: $\mathbf{H}_t^l = \mathbf{H}_{t-1}^l + \mathbf{X}_t^{l\top} \mathbf{X}_t^l$
		\STATE \quad Compute gain matrix: $\mathbf{S}_t^l = (\mathbf{H}_t^l)^{-1} (\mathbf{X}_t^{l\top} \mathbf{X}_t^l)$
		\STATE \quad Recursive parameter merging: $\tau_t^{l*} = \tau_{t-1}^{l*} + \mathbf{S}_t^l (\tau_t^l - \tau_{t-1}^{l*})$
		\ENDFOR
		\ENDFOR

		\STATE \textbf{Inference Phase:}
		\STATE Extract global feature $\mathbf{Z}_v^{\text{g}}$ and spatial tokens $\mathbf{Z}_v^{\text{s}}$ from $f_v(x_v; \Phi_v)$
		\STATE Compute similarity between $\mathbf{Z}_v^{\text{g}}$ and $\mathbf{\mu}_i$ via Eq.~\ref{eq:similarity}
		\STATE Calculate merged projected tokens $\mathbf{P}_v$ according to Eq.~\ref{eq:merging_projection}
		\STATE Construct reasoning weights $\{W_{\text{final}}^l\}_{l=1}^L$ where $W_{\text{final}}^l = W_{\text{pre}}^l + \lambda (\tau_n^{l*})^\top$
		\STATE \textbf{return} $\mathbf{P}_v$ and $\{W_{\text{final}}^l\}_{l=1}^L$ to generate answer
	\end{algorithmic}
\end{algorithm}

Defining the gain matrix as $\mathbf{S}_t^l = (\mathbf{H}_t^l)^{-1} (\mathbf{X}_t^{l\top} \mathbf{X}_t^l)$, we arrive at the final recursive fusion rule:
\begin{equation}
\tau_t^{l*} = \tau_{t-1}^{l*} + \mathbf{S}_t^l (\tau_t^l - \tau_{t-1}^{l*}).
\end{equation}
This derivation rigorously confirms that the incremental update is mathematically identical to the joint batch optimum. Crucially, this equivalence is maintained without requiring the raw historical features $\mathbf{X}_{1:t-1}^l$ storage, as all necessary information is compressed within the sufficient statistic $\mathbf{H}_t^l$.

\section{Theoretical Mechanism and Subspace Analysis of Low-Rank Parameter Merging}
\label{sec:appendix_mechanism}
As established in Appendix~\ref{sec:appendix_rls}, our recursive task vector fusion rule yields an exact mathematical equivalent to the joint batch optimal solution $\tau_t^{l*} = ( \sum_{i=1}^t \mathbf{X}_i^{l\top} \mathbf{X}_i^l )^{-1} \sum_{i=1}^t \mathbf{X}_i^{l\top} \mathbf{X}_i^l \tau_i^l$. To elucidate the underlying mechanics of this approach, we present a theoretical analysis focusing on the fusion behavior of matrix-multiplication parameters.Specifically, we apply SVD to the input feature matrix of a given task $i$: $\mathbf{X}_i^l = \mathbf{U}_i^l \boldsymbol{\Sigma}_i^l \mathbf{V}_i^{l\top}$. To rigorously handle potential rank deficiency in feature spaces, we employ the Moore-Penrose pseudoinverse ($^\dagger$) in the subsequent derivations. By analyzing two extreme cases of feature distribution, we uncover the fusion dynamics at play.

\textbf{Ideal case (Mutually orthogonal features).}  Assume that for any pair of distinct tasks $i \neq j$, their right singular vectors satisfy the strict orthogonality condition ($\mathbf{V}_i^{l\top} \mathbf{V}_j^l = \mathbf{0}$). Under this condition, the batch optimal solution simplifies significantly to the summation of task vectors projected onto their respective feature subspaces:
\begin{equation}
\begin{aligned}
\tau_{t,\text{ideal}}^{l*} &= \left( \sum_{i=1}^t \mathbf{V}_i^l (\boldsymbol{\Sigma}_i^l)^2 \mathbf{V}_i^{l\top} \right)^\dagger \sum_{j=1}^t \mathbf{V}_j^l (\boldsymbol{\Sigma}_j^l)^2 \mathbf{V}_j^{l\top} \tau_j^l \\
&= \sum_{i=1}^t \left( \mathbf{V}_i^l (\boldsymbol{\Sigma}_i^l)^2 \mathbf{V}_i^{l\top} \right)^\dagger \left( \mathbf{V}_i^l (\boldsymbol{\Sigma}_i^l)^2 \mathbf{V}_i^{l\top} \right) \tau_i^l \\
&= \sum_{i=1}^t \mathbf{V}_i^l \mathbf{V}_i^{l\top} \tau_i^l.
\end{aligned}
\end{equation}
Here, $\mathbf{V}_i^l \mathbf{V}_i^{l\top} \tau_i^l$ represents the orthogonal projection of the task vector $\tau_i^l$ onto the subspace spanned by the singular vectors $\mathbf{V}_i^l$. This projection mechanism exclusively retains the components of $\tau_i^l$ that align with the corresponding feature space $\mathbf{X}_i^l$, effectively filtering out irrelevant update directions.Consequently, in this ideal orthogonal setting, where $\mathbf{V}_i^{l\top} \mathbf{V}_j^l = \mathbf{0}$ for any $i \neq j$, the optimal fusion strategy achieves lossless knowledge integration. We prove that the cumulative layer-wise feature drift remains strictly zero by substituting the ideal solution $\tau_{t,\text{ideal}}^{l*} = \sum_{j=1}^t \mathbf{V}_j^l \mathbf{V}_j^{l\top} \tau_j^l$ into the objective:
\begin{equation}
\begin{adjustbox}{max width=1.0\linewidth}
$\displaystyle
\begin{aligned}
\sum_{i=1}^t \| \mathbf{X}_i^l
    (\tau_{t,\text{ideal}}^{l*} - \tau_i^l) \|_F^2
&= \sum_{i=1}^t \left\|
    \mathbf{U}_i^l \boldsymbol{\Sigma}_i^l \mathbf{V}_i^{l\top}
    \left( \sum_{j=1}^t \mathbf{V}_j^l \mathbf{V}_j^{l\top}
    \tau_j^l - \tau_i^l \right)
    \right\|_F^2 \\
&= \sum_{i=1}^t \|
    \mathbf{U}_i^l \boldsymbol{\Sigma}_i^l
    \sum_{j=1}^t (\mathbf{V}_i^{l\top} \mathbf{V}_j^l)
    \mathbf{V}_j^{l\top} \tau_j^l
    - \mathbf{U}_i^l \boldsymbol{\Sigma}_i^l
    \mathbf{V}_i^{l\top} \tau_i^l \|_F^2 \\
&= \sum_{i=1}^t \|
    \mathbf{U}_i^l \boldsymbol{\Sigma}_i^l
    \mathbf{V}_i^{l\top} \tau_i^l
    - \mathbf{U}_i^l \boldsymbol{\Sigma}_i^l
    \mathbf{V}_i^{l\top} \tau_i^l \|_F^2 = 0.
\end{aligned}$
\end{adjustbox}
\end{equation}
\textbf{Worst case (Fully collinear features).} Conversely, we consider the opposite extreme where all task features share an identical set of right singular vectors, i.e., $\mathbf{V}_i^l = \mathbf{V}^l$ for all $i \in \{1, \dots, t\}$. This scenario represents a state of maximum interference, where tasks compete for the same parameter directions. Under this highly overlapping condition, the batch optimal solution degenerates into an adaptive weighted average within the shared subspace:
\begin{equation}
\begin{aligned}
\tau_{t,\text{worst}}^{l*} &= \left( \sum_{i=1}^t \mathbf{V}^l (\boldsymbol{\Sigma}_i^l)^2 \mathbf{V}^{l\top} \right)^\dagger \sum_{i=1}^t \mathbf{V}^l (\boldsymbol{\Sigma}_i^l)^2 \mathbf{V}^{l\top} \tau_i^l \\
&= \mathbf{V}^l \left( \sum_{j=1}^t (\boldsymbol{\Sigma}_j^l)^2 \right)^\dagger \mathbf{V}^{l\top} \sum_{i=1}^t \mathbf{V}^l (\boldsymbol{\Sigma}_i^l)^2 \mathbf{V}^{l\top} \tau_i^l \\
&= \sum_{i=1}^t \mathbf{V}^l \left[ \left( \sum_{j=1}^t (\boldsymbol{\Sigma}_j^l)^2 \right)^\dagger (\boldsymbol{\Sigma}_i^l)^2 \right] \mathbf{V}^{l\top} \tau_i^l.
\end{aligned}
\end{equation}
Here, the term $\mathbf{V}^l [ ( \sum_{j=1}^t (\boldsymbol{\Sigma}_j^l)^2 )^\dagger (\boldsymbol{\Sigma}_i^l)^2 ] \mathbf{V}^{l\top}$ acts as a normalized importance weight in the spectral domain. 
Unlike the ideal case, feature drift becomes inevitable due to the structural overlap of task features. However, our recursive formulation ensures that the merged vector $\tau_t^{l*}$ is not a naive arithmetic average, but an optimal synthesis that prioritizes task components with higher spectral energy. By adaptively weighting updates according to their respective singular values, LPM ensures that the parameter directions most critical to historical knowledge are preserved with high fidelity, even under severe collinearity. This mechanism effectively minimizes the cumulative reconstruction error of task features within the shared parameter subspace, thereby mitigating catastrophic interference and safeguarding the structural integrity of the model's reasoning capabilities.

These cases theoretically show that the cumulatively optimized $\tau_t^{l*}$ strikes an intrinsic balance between task-specific specialization and shared structural patterns. By locally projecting independent transformations within orthogonal subspaces while robustly synthesizing overlapping features, LPM yields a theoretically-grounded consolidated vector that encapsulates the essence of the task stream.

\section{Implementation Details}
\label{sup_implementation}
On the UCIT benchmark, all methods are trained for 1 epoch with a batch size of 16. The learning rates for LLaVA and InternVL are set to 2e-4 and 1e-4, respectively. On the MLLM-DCL benchmark, the batch size is reduced to 8, with learning rates of 2e-5 for LLaVA and 1e-4 for InternVL. Training epochs for this benchmark are task-specific: 3 epochs for the Medicine task, 2 epochs for the Science task, and 1 epoch for all remaining tasks.

\begin{figure}[!t]
    \centering
    \includegraphics[width=\columnwidth]{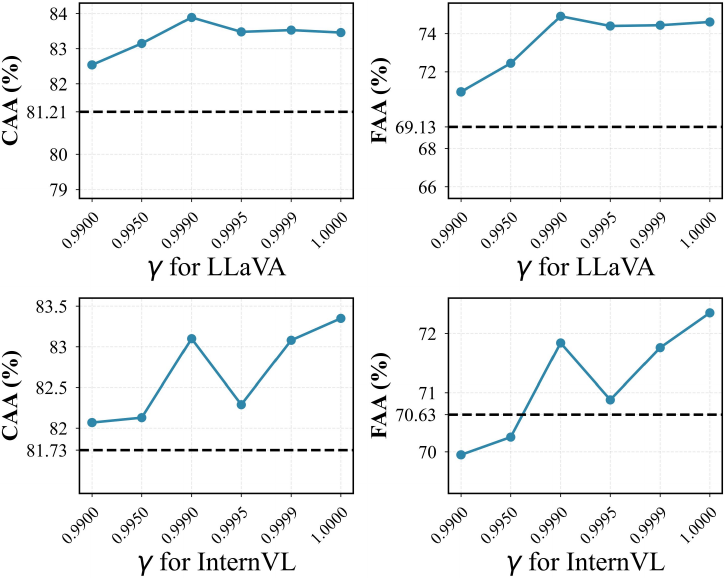}
    \caption{Sensitivity to $\gamma$. Dashed lines indicate the best baseline results for the corresponding metrics.}
    \label{fig:gamma_cropped}
\end{figure}

\begin{table*}[!t]
    \caption{
    Under the two settings of the projector corresponding to the task and the merged projector of CPM, the impact of using the corresponding task's LoRA and the LPM's Merged LoRA and the last task LoRA using LLaVA-1.5-7B.
    }
    \label{tab:lora_projector_comparison}
    \centering
    \renewcommand{\arraystretch}{0.72} 
    \setlength{\aboverulesep}{0.25ex}
    \setlength{\belowrulesep}{0.25ex}
    \setlength{\tabcolsep}{3.5pt} 
    \resizebox{\textwidth}{!}{%
    \begin{tabular}{@{} l l c c c c c c c c c @{}}
        \toprule
        \multirow{2}{*}{\textbf{Projector Setting}} & \multirow{2}{*}{\textbf{LoRA Setting}} & \multicolumn{6}{c}{\textbf{UCIT}} & \multicolumn{3}{c}{\textbf{Metrics}} \\
        \cmidrule(lr){3-8} \cmidrule(l){9-11}
        & & Img-R & Arxiv & VizCap & Icon & CLEVR & F30k & FAA ($\uparrow$) & CAA ($\uparrow$) & FFM ($\downarrow$) \\
        \midrule
        \multirow{3}{*}{\textbf{Task Projector}} 
        & Task LoRA (Oracle)  & \textbf{91.70} & \textbf{90.90} & \textbf{61.78} & \textbf{85.30} & \textbf{78.31} & 57.98 & \textbf{77.66} & - & - \\
        & Merged LoRA       & 88.83 & 89.43 & 56.29 & 75.43 & 58.57 & 56.77 & 70.89 & \textbf{82.05} & \textbf{5.92} \\
        & Final Task LoRA & 81.20 & 76.40 & 46.49 & 78.63 & 74.70 & \textbf{57.99} & 69.24 & 77.80 & 10.08 \\
        \midrule
        \multirow{3}{*}{\textbf{Merged Projector}} 
        & Task LoRA  & \textbf{91.63} & 91.10 & \textbf{60.79} & \textbf{83.07} & \textbf{77.17} & 58.13 & \textbf{76.98} & \textbf{84.01} & \textbf{0.23} \\
        & Merged LoRA        & 90.87 & \textbf{91.97} & 60.48 & 80.20 & 67.37 & 56.09 & 74.50 & 83.42 & 0.37 \\
        & Final Task LoRA & 78.43 & 76.73 & 46.03 & 79.30 & 73.50 & \textbf{58.28} & 68.71 & 75.65 & 10.45 \\
        \bottomrule
    \end{tabular}%
    }
\end{table*}

\begin{table*}[!t]
    \caption{Comparison of modality choices for CPM prototype construction on the UCIT benchmark.}
    \label{tab:cpm_construction}
    \centering
    \renewcommand{\arraystretch}{0.78}
    \setlength{\aboverulesep}{0.3ex}
    \setlength{\belowrulesep}{0.3ex}
    \resizebox{\textwidth}{!}{%
    \begin{tabular}{@{} l c c c c c c c c c @{}}
        \toprule
        \multirow{2}{*}{\textbf{Modality}} & \multicolumn{6}{c}{\textbf{UCIT}} & \multicolumn{3}{c}{\textbf{Metrics}} \\
        \cmidrule(lr){2-7} \cmidrule(l){8-10}
        & Img-R & Arxiv & VizCap & Icon & CLEVR & F30k & FAA ($\uparrow$) & CAA ($\uparrow$) & FFM ($\downarrow$) \\
        \midrule
        Text  & \textbf{91.20} & \textbf{92.67} & 57.81 & \textbf{80.70} & 67.67 & \textbf{56.32} & 74.40 & \textbf{83.48} & 0.70 \\
        Text + Image   & 91.03 & 92.63 & 58.89 & 80.23 & \textbf{68.57} & 56.29 & \textbf{74.61} & 83.46 & 0.38 \\
        Image (Ours)   & 90.87 & 91.97 & \textbf{60.48} & 80.20 & 67.37 & 56.09 & 74.50 & 83.42 & \textbf{0.37} \\
        \bottomrule
    \end{tabular}%
    }
\end{table*}

\begin{table*}[!t]
    \caption{Comparison of modality choices for CPM prototype construction on the MLLM-DCL benchmark.}
    \label{tab:cpm_construction_mllm_dcl}
    \centering
    \renewcommand{\arraystretch}{0.78}
    \setlength{\aboverulesep}{0.3ex}
    \setlength{\belowrulesep}{0.3ex}
    \resizebox{0.8\textwidth}{!}{%
    \begin{tabular}{@{} l c c c c c c c c @{}}
        \toprule
        \multirow{2}{*}{\textbf{Modality}} & \multicolumn{5}{c}{\textbf{MLLM-DCL}} & \multicolumn{3}{c}{\textbf{Metrics}} \\
        \cmidrule(lr){2-6} \cmidrule(l){7-9}
        & RS & Med & AD & Sci & Fin & FAA ($\uparrow$) & CAA ($\uparrow$) & FFM ($\downarrow$) \\
        \midrule
        Text         & 79.69 & 57.19 & 47.61 & 48.12 & 89.38 & 64.40 & 65.62 & \textbf{-2.27} \\
        Text + Image & \textbf{79.81} & \textbf{59.20} & 47.50 & 49.23 & \textbf{89.46} & \textbf{65.04} & \textbf{66.46} & -1.78 \\
        Image (Ours) & 79.72 & 58.85 & \textbf{47.63} & \textbf{49.39} & 89.43 & 65.00 & 66.43 & -1.61 \\
        \bottomrule
    \end{tabular}%
    }
\end{table*}

\section{Sensitivity of the Energy Threshold $\gamma$}
\label{sensiety_gamma}

Figure~\ref{fig:gamma_cropped} evaluates JoMe$^*$ across varying energy thresholds, where dashed lines represent the SOTA. JoMe$^*$ consistently outperforms the SOTA across the entire range. Notably, a performance peak emerges at $\gamma = 0.999$, which surprisingly surpasses the uncompressed full-rank setting ($\gamma = 1.0$). This indicates that truncating long-tail singular values acts as an implicit regularizer by filtering noisy directions that exacerbate cross-task interference. This optimal configuration reduces storage overhead to merely 1/27 of the original requirement while achieving superior accuracy and stability. We therefore set $\gamma = 0.999$ as the default for JoMe$^*$.

\section{The Performance Gap and Synergy between Perceptual and Reasoning Merging}
\label{sec:appendix_lpm_bound}
To deeply evaluate the knowledge consolidation capability of LPM, we compare it against two extreme settings in Table~\ref{tab:lora_projector_comparison}: (1) \textbf{Task LoRA}, which serves as an ideal Oracle upper bound requiring ground-truth task IDs at test time, and (2) \textbf{Final Task LoRA}, representing the lower bound where only the LoRA weights from the last trained task are used for inference. Under the CPM setting, LPM (FAA $74.50\%$, FFM $0.37\%$) demonstrates exceptional performance retention, closely approximating the Oracle Task LoRA (FAA $76.98\%$, FFM $0.23\%$) while vastly outperforming the naive Final LoRA integration. This establishes that the recursive least squares parameter merging mechanism in LPM can almost perfectly parse and preserve the structural reasoning weights of sequential tasks, even under strict exemplar-free and task-agnostic constraints.

\section{Modality Choices for CPM Prototype Construction}
\label{sec:appendix_cpm_modality}

In Section~\ref{sec:method_routing}, we opted to construct the task-representative prototypes $\mathbf{\mu}_t$ using solely visual features. We ablate this design choice in Table~\ref{tab:cpm_construction} and Table~\ref{tab:cpm_construction_mllm_dcl} by comparing prototypes built from text-only instructions, fused text-image features, and image-only features. To extract these multimodal embeddings, we consistently adopt CLIP-L/14-336~\cite{radford2021learning}. In Table~\ref{tab:cpm_construction}, the results indicate highly consistent overall performance across all three modality choices, with FAA ranging tightly between 74.4\% and 74.6\% alongside minimal forgetting. This implies that the tasks exhibit sufficient distinctiveness in both their visual distributions and textual instructions. However, the image-only approach presents a distinct computational advantage: during inference, it completely frees the model from forwarding textual instructions through an additional text encoder to execute routing decisions. Consequently, our vision-centric design maximizes computational efficiency, ultimately yielding the optimal FFM of 0.37\% on UCIT benchmark.

\begin{table*}[!t]
    \caption{Comparison of JoMe with model merging methods on UCIT using LLaVA-1.5-7B. Bold indicates the best result in each column.}
    \label{tab:ucit_model_merging}
    \centering
    \renewcommand{\arraystretch}{0.78}
    \setlength{\aboverulesep}{0.3ex}
    \setlength{\belowrulesep}{0.3ex}
    \resizebox{\textwidth}{!}{%
    \begin{tabular}{@{} l c c c c c c c c c @{}}
        \toprule
        \multirow{2}{*}{\textbf{Method}}
        & \multicolumn{6}{c}{\textbf{UCIT}}
        & \multicolumn{3}{c}{\textbf{Metrics}} \\
        \cmidrule(lr){2-7} \cmidrule(l){8-10}
        & Image-R & ArxivQA & Viz-cap & IconQA & CLEVR & Flickr30k
        & FAA ($\uparrow$) & CAA ($\uparrow$) & FFM ($\downarrow$) \\
        \midrule
        Task Arithmetic~\cite{ilharco2022editing}
        & 88.97 & 79.57 & 49.73 & \textbf{81.40} & 68.67 & 56.66 & 70.83 & 80.24 & 0.49 \\
        TIES-Merging~\cite{yadav2023ties}
        & 89.80 & 79.67 & 48.86 & 79.13 & 70.80 & 56.82 & 70.85 & 80.52 & 2.02 \\
        DARE~\cite{yu2024language}
        & 89.03 & 79.43 & 50.60 & 80.97 & 68.07 & 56.54 & 70.77 & 80.25 & 0.65 \\
        MagMax~\cite{marczak2024magmax}
        & 87.80 & 78.90 & 49.51 & 79.10 & \textbf{73.13} & \textbf{58.01} & 71.08 & 80.61 & 3.53 \\
        Core Space~\cite{panariello2509accurate}
        & 89.77 & 79.80 & 49.98 & 78.40 & 67.20 & 56.39 & 70.26 & 80.44 & 3.02 \\
        \midrule
        \textbf{JoMe (Ours)}
        & \textbf{90.87} & \textbf{91.97} & \textbf{60.48} & 80.20 & 67.37 & 56.09
        & \textbf{74.50} 
        & \textbf{83.42} 
        & \textbf{0.37} \\
        \bottomrule
    \end{tabular}%
    }
\end{table*}

\section{Comparison with Model Merging Methods} \label{sec:appendix_mm} 
Table~\ref{tab:ucit_model_merging} compares JoMe with five model merging baselines on UCIT using LLaVA-1.5-7B. All methods use the same CPM module for projector merging, and we vary only how the task-specific LoRA updates are merged. For each baseline, we retain the task-specific LoRA checkpoints and rerun batch merging at each evaluation stage using all checkpoints available at that stage. JoMe updates its merged LoRA parameters incrementally after each task. JoMe achieves an FAA of 74.50\% and a CAA of 83.42\%, improving over MagMax, the strongest baseline on both metrics, by 3.42 and 2.81 \%, respectively. JoMe also achieves the lowest FFM at 0.37\%, which is 0.28 \% lower than DARE, the strongest baseline on this metric. These results show that JoMe improves accuracy both at the end of the task sequence and on average across learning stages, while reducing forgetting. JoMe achieves the highest final accuracy on Image-R, ArxivQA, and Viz-cap. On ArxivQA and Viz-cap, JoMe outperforms the best baseline for each task by 12.17 and 9.85 \%, respectively. Task Arithmetic leads on IconQA, while MagMax~\cite{marczak2024magmax} leads on CLEVR and Flickr30k.

\end{document}